\documentclass[12pt]{article}
\usepackage[left = 1.25in, right = 1.25in]{geometry}
\usepackage[dvipsnames]{xcolor}

\usepackage{natbib}
\usepackage{amscd,amssymb,amsmath,mathrsfs,amsfonts,graphicx,verbatim,epsfig,
psfrag, float, fancybox} 
\usepackage{graphics}                 
\usepackage{color}                    
\usepackage{hyperref}                 
\usepackage[TS1,OT1,T1]{fontenc}

\usepackage{amsmath}    
\usepackage{amscd,amssymb,mathrsfs,amsfonts}
\usepackage{graphicx}   
\usepackage{verbatim}   
\usepackage{color}      
\usepackage{amsmath}
\usepackage{verbatim}
\usepackage{setspace} 
\usepackage{algorithm}

\usepackage{authblk}

\newcommand*{\m}{\boldsymbol{\mu}}
\newcommand*{\thet}{\boldsymbol{\Theta}}

\newcommand*{\bb}{\beta\beta}
\newcommand*{\hbb}{\widehat{\beta\beta}}

\newcommand{\Keywords}[1]{\par\noindent
{\small{\em Keywords\/}: #1}}

\newtheorem{theorem}{Theorem}[section]

\newtheorem{proof}[theorem]{Proof}

\begin{document}

\title{On Estimating Many Means, Selection Bias, and the Bootstrap}

\author{Noah Simon \thanks{Department of biostatistics, University of Washington, \texttt{nrsimon@uw.edu}} \and Richard Simon\thanks{NIH, National Cancer, Institute Biometric Research Branch}}

\maketitle

\begin{abstract}
With recent advances in high throughput technology, researchers often find themselves running a large number of hypothesis tests (thousands+) and estimating a large number of effect-sizes. Generally there is particular interest in those effects estimated to be most extreme. Unfortunately naive estimates of these effect-sizes (even after potentially accounting for multiplicity in a testing procedure) can be severely biased. In this manuscript we explore this bias from a frequentist perspective: we give a formal definition, and show that an oracle estimator using this bias dominates the naive maximum likelihood estimate. We give a resampling estimator to approximate this oracle, and show that it works well on simulated data. We also connect this to ideas in empirical Bayes.

\end{abstract}

\doublespace

\section{Introduction}\label{sec:intro}
Often, in modern applications, researchers are interested in testing and estimating effect sizes for many different features at once. In the simplest cases one is interested in estimating population means from a sample (often with the most extreme means as the most interesting).\\

In his revolutionary paper \citep{stein1956}, Charles Stein showed that, for estimating the means of three or more Gaussian random variables, one can do better (in terms of MSE) than simply using the sample means --- that cleverly shrinking the effect sizes will strictly dominate the obvious estimate. \citet{efron1973} illustrate this on baseball batting averages and give some insight into the phenomenon. If the population means are very close together, then the sample means will be too spread out --- the largest sample mean likely came from a large population mean, but also got lucky and won the ``largest statistic'' competition, so it is biased high (similarly true for the smallest sample mean, though biased low).\\

This regression-to-the-mean type idea is concerning for estimating effect sizes in high-throughput experiments. Generally scientists look at the most extreme effects in the sample and report the unadjusted sample estimates of these effect-sizes. As such, on attempts to confirm these effects, effect sizes (even of very statistically significant effects) are often much less extreme than originally reported.\\

This selection bias has been explored using ideas in empirical Bayes and compound decision theory (\citet{robbins1956}, \citet{robbins1985} among others), though much of this focus was on asymptotically sub-minimax estimators, rather than selection bias in particular, and before the explosion of high dimensional data and modern computing, this was more a theoretical than practical pursuit.

With new important applications and cheap computing, these problems have gained popularity. Recently \citet{efron2011} gave an elegant approach to correct for this selection bias by applying empirical Bayes via Tweedie's formula. Unlike the estimator of \citet{stein1961}, which shrinks everything toward the overall mean (ignoring all information beyond the square sum of the statistics), Efron's formulation gives locally adaptive shrinkage (shrinking based on the local shape of the histogram of statistics). It is particularly appealing as it requires no parametric assumptions on the prior. These ideas have been extended by (\citet{wager2013}, \citet{jiang2009}, \citet{brown2009}, among others), with very efficient estimates of the marginal likelihood.

In this paper, we give a frequentist formulation of the selection bias problem. We show how this bias affects mean square error, and give intuition for classical frequentist shrinkage ideas. We also discuss connections between this ``frequentist selection bias'' and the more common Bayesian shrinkage in the literature.

Motivated by our definition of frequentist selection bias, we give a simple procedure based on a parametric bootstrap to estimate the frequentist bias. We compare this bootstrap approach to several flavors of empirical Bayes. In situations where empirical Bayes is applicable, the two perform comparably (though empirical Bayes is slightly stronger), however we also detail many situations where empirical Bayes solutions are intractable while our bootstrap shrinkage is simple and effective.

\section{Selection Bias}
Suppose we have a large number ($p$) of features, and for each feature ($i$) we have a sample estimate ($z_i$) of its mean ($\mu_i$) which is normally distributed with variance $1$
\[
z_i \sim N\left(\mu_i,1\right)
\]
This is approximately the scenario we get from many t-tests (as in \citet{tusher2001} and others). Now, clearly for a fixed $i$, $z_i$ is an unbiased estimate of $\mu_i$. Generally however we select the largest $z_i$ (or $|z_i|$) and would like to estimate its corresponding mean. Because we have selected an extreme statistic, if we use the unadjusted statistic as an estimate of the mean, we incur a selection bias.\\

 To explore this bias let us first introduce some notation. Let $z_{[k]}$ denote the $k$-th order statistic, and $i(k)$ denote the index of the $k$-th order statistic (i.e., $z_{i(k)} = z_{(k)}$). Note that since the ordering of our statistics is stochastic, $i(k)$ is a \emph{stochastic} index (the inverse rank of $z_{(k)}$).\\
The bias we are interested in is 
\begin{equation}\label{shrink:c}
\operatorname{E}\left[z_{(1)} - \mu_{i(1)}\right]
\end{equation}
Note, both the order statistic $z_{[1]}$ and the mean $\mu_{i(1)}$ are random variables (the mean is a random variable because the index $i(1)$ is stochastic).

For each rank, $k$, we can use the same idea and define our bias as 
\begin{equation}\label{shrink:C}
\beta_k = \operatorname{E}\left[z_{(k)} - \mu_{i(k)}\right]
\end{equation}

If we knew these biases (unrealistic in practice) then we could estimate the means of the extreme statistics by
\begin{equation}\label{eq:oracle}
\tilde{\mu}_{i(k)} = z_{(k)} - \beta_k
\end{equation}
This estimate dominates the naive estimate, $\hat{\mu}_i = z_i$, in terms of $\ell_2$ loss. This is straightforward, as
\begin{align*}
\operatorname{E}\left[\sum_{i=1}^p \left(\hat{\mu}_i - \mu_i\right)^2\right] &= \operatorname{E}\left[\sum_{k=1}^p \left(\hat{\mu}_{i(k)} - \mu_{i(k)}\right)^2\right]\\
&= \sum_{k=1}^p \operatorname{E}\left[\left(z_{(k)} - \mu_{i(k)}\right)^2\right]\\
&=\sum_{k=1}^p \left[\beta_k^2 + \operatorname{var}\left(z_{[(k)} - \mu_{i(k)}\right)\right]\\
&=\sum_{k=1}^p \left[\beta_k^2 + \operatorname{var}\left(z_{(k)} - \beta_k - \mu_{i(k)}\right)\right]\\
&=\sum_{k=1}^p \beta_k^2 + \operatorname{E}\left[\sum_{k=1}^p \left(\tilde{\mu}_{i(k)} - \mu_{i(k)}\right)^2\right]
\end{align*}
Our risk decreases by the sum of the squared biases. For the remainder of the manuscript, we will refer to the estimates in Eq~\eqref{eq:oracle}, with the true biases known, as our ``oracle estimates.''\\

For ease of reading we will define the following notation. We will use $\m$ to refer to a vector of means, with $\mu_i$ to denote the $i$-th element of $\m$. Let $\beta\left(\m\right)$ denote the vector of biases for a given mean vector $\m$. More specifically
\[
\beta\left(\m\right)_k = \operatorname{E}\left[z_{(k)} - \mu_{i(k)}\right].
\]

\subsection{Estimating the Bias}

In practice we will never know the bias \eqref{shrink:C} and must estimate it. We propose a simple $2$ step method. We first estimate $\m$ by maximum likelihood, giving, in this case $\hat{\m} = z$. We then use the biases for this estimated model, as estimates of the bias for our original model
\begin{equation}\label{eq:b}
\hat{\beta}\left(\m\right) = \beta\left(\hat{\m}\right)
\end{equation}
Finally, as our updated estimate of $\m$ we use $\tilde{\mu}_{i(k)} = \hat{\mu}_{i(k)} - \hat{\beta}_{i(k)} = z_{(k)} - \hat{\beta}_{i(k)}$. This is just a parametric bootstrap.

\subsection{Calculating the Bias of the Estimated Model}

Now, given known means $\mu_i$, $i = 1,\ldots,p$ we need to calculate the bias. This is most tractable by monte-carlo. Though the monte-carlo is straightforward, we give it in full detail.
\begin{enumerate}
\item for $b=1,\ldots,B$
\begin{enumerate}
\item Simulate $z_1^b,\ldots,z_p^b$ a $p$-vector of Gaussians with means $\mu_1,\ldots,\mu_p$ and variance 1
\item Find the $k$-th order statistic $z_{(k)}^b$ and the index of its corresponding mean $i(k)^b$
\end{enumerate}
\item Calculate the $k$th bias as 
\[
\frac{1}{B}\sum_{b=1}^B \left(z_{(k)}^b - \mu_{i(k)^b}\right)
\]
\end{enumerate}
As $B$ grows, this will give increasingly accurate calculations of the bias.\\

\subsection{Second Order Bias}
In many cases we can improve the estimate in \eqref{eq:b}. Heuristically the further spread out the means are, the smaller the bias is. Intuitively, our sample means are more spread out than the true means --- thus our estimates of the bias, are themselves biased. Often, the bias estimate in the bootstrap sample will be smaller than the true bias. 

Formally, this second-order bias is
\[
\bb\left(\m\right) = \operatorname{E}\left[\beta\left(\m\right) - \beta\left(\hat{\m}\right)\right].
\]
If we knew this quantity, then we could update our estimate $\hat{\beta}\left(\m\right)$ to
\[
\hat{\beta}^{(2)}\left(\m\right) = \hat{\beta}\left(\m\right) + \bb\left(\m\right)
\]
Unfortunately, this quantity is never known in practice. To approximate it, we use the same trick as before
\begin{equation}\label{eq:bb}
\hbb\left(\m\right) = \bb\left(\hat{\m}\right)
\end{equation}
and use the estimate
\begin{equation}\label{eq:second-order}
\hat{\beta}^{(2)}\left(\m\right) = \hat{\beta}\left(\m\right) + \hbb\left(\m\right)
\end{equation}
It is straightforward to calculate the estimate of second-order bias in \eqref{eq:bb} via monte-carlo, though, for the sake of brevity, we leave the details to the reader. For the remainder of the manuscript we will refer to 
\begin{align}
\tilde{\mu}_{i(k)} &= z_{(k)} - \hat{\beta}_k\\
\tilde{\mu}^{(2)}_{i(k)} &= z_{(k)} - \hat{\beta}^{(2)}_k
\end{align}
as the ``first-order'' and ``second-order'' bootstrap estimates of effect-size.

This second-order bias estimate $\hbb\left(\mu\right)$ will also be biased for the true second-order bias $\bb\left(\mu\right)$. One might consider higher order de-biasing. There is a ``bias-variance'' tradeoff here, and in practice, while the second-order correction has been useful in some problems, we have not seen improvement past second-order corrections.

\subsection{Simple Example}
We will give a simple example illustrating the difference between the ``local'' shrinkage of our method and the ``global'' shrinkage of James-Stein estimation. In our example we simulate $1000$ features, $990$ of which have $\mu_i = 0$; the remaining $10$ have $\mu_i = 6$. We can see the results in Figure~\ref{fig:simple}. James-Stein gives great shrinkage for the bulk of the features, however it \emph{completely} shrinks away the interesting effects. In contrast, our resampling approach can take local behavior into account, and while it shrinks the bulk of the effect sizes to $0$, it correctly pushes the estimates for the interesting effect towards $6$.
 
\begin{figure}
  \begin{center}
    \includegraphics[scale=0.9]{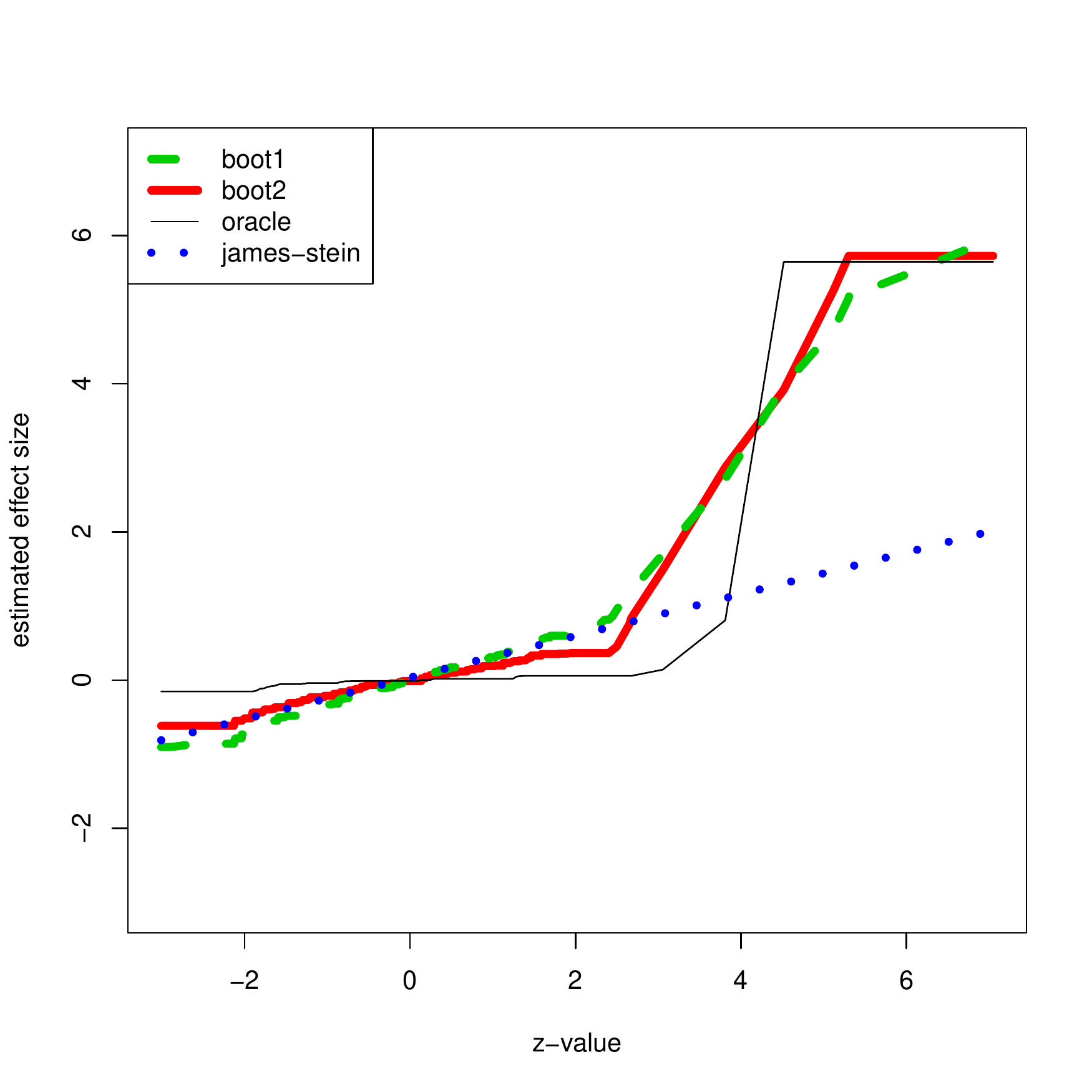}
  \end{center}
  \caption{Plots showing shrinkage, with ordered $z$-values on the x-axis vs estimated means on the y-axis. The blue dotted line is the James-Stein estimate. Green dashed and red solid lines are the [smoothed] first and second order bootstrap estimates. The skinny black line is the oracle estimate, using the biases calculated from the true means.}
  \label{fig:simple}
\end{figure}

\subsection{Empirical Bayes}
We will briefly review the empirical Bayes approach of \citet{efron2011}. If we assume some prior density ($g$) on the means
\[
\mu\, \sim\, g(\cdot) \textrm{ with } z|\mu\, \sim\, N\left(\mu,1\right)
\]
and let $f(z)$ denote the marginal distribution of $z$
\[
f(z) = \int \phi\left(z - \mu\right)g(\mu)d\mu
\]
then the posterior expectation of $\mu$ given $z$ is
\begin{equation}\label{eq:eBayes}
\operatorname{E}\left[\mu \middle | z\right] = z + f^{'}(z)/f(z)
\end{equation}
This result is known as Tweedie's theorem \citep{efron2011}. We will term $f^{'}(z)/f(z)$ as the Bayesian bias --- it adjusts the naive estimate and accounts for selection bias (this will be made more clear in Section~\ref{sec:rec}). This approach is elegant as \eqref{eq:eBayes} does not require $g$, or a direct estimate of $g$. Instead, one needs only estimate $f$: a smoothed histogram of the $z_i$  (and its derivative). This is tractable if the number of features is large --- though estimating the derivative is still difficult, and the degree of smoothing can influence results.\\

Efron suggests using Lindsey's method \citep{efron1996} for the density estimate. For Lindsey's method, one bins the data, and uses poisson regression with a spline or polynomial basis and an offset for bin-size to estimate the density. More recently \citet{wager2013} and others have given nonparametric approaches which generally outperform Lindsey-based approaches. In our comparisons in Section~\ref{sec:emp}, we include the estimate of \citet{wager2013}, which we will refer to as \texttt{nlpden}.

\subsection{Reconciling the differences}\label{sec:rec}
We have two approaches to selection bias (frequentist and Bayesian) that at first glance are very different however on deeper consideration they are actually quite similar. The Bayesian bias that we estimate in the Bayesian approach is
\begin{equation}
z_i - \operatorname{E}\left[\mu\middle|z_i\right]
\end{equation}
whereas in the frequentist approach the bias is
\begin{equation}\label{eq:f}
\operatorname{E}\left[z_{[k]} - \mu_{i(k)}\right]
\end{equation}
Though frequentist, \eqref{eq:f} already has a Bayesian flavor as $i(k)$, the index of our mean, is stochastic. In the Bayesian framework (where $\mu$ has some prior distribution), we can show that (under some assumptions) these two biases are asymptotically the same.

\begin{theorem}\label{theorem:Bayes}
Let $G(\cdot)$ be a probability measure on $\mathbb{R}$ with bounded support. Assume, for $i=1,\ldots,p$, 
\[
\mu_i\,\overset{iid}{\sim}\, G(\cdot).
\]
Then, as $p\rightarrow\infty$, for any $t\in (0,1)$ we have
\[
\operatorname{E}\left[z_{[\lfloor tp\rfloor]} - \mu_{i(\lfloor tp\rfloor)}\right] \rightarrow F^{-1}\left(t\right) - \operatorname{E}\left[\mu \middle| z = F^{-1}\left(t\right)\right]
\]
where $\lfloor \cdot \rfloor$ is the ``floor'' function and $F^{-1}\left(t\right)$ is the $t$-th quantile of the marginal distribution $dF(z) = \int \phi\left(z - \mu\right)dG(\mu)$.
\end{theorem}
The proof is given in the appendix. The assumption of bounded support for $G$ can easily be weakened, but it simplifies the proof.

This raises a similar question in the frequentist framework. Namely, if we denote the empirical distribution of the means by $G_p$:
\[
G_p\left(\mu\right) = \frac{1}{p}\sum_i I\left\{\mu = \mu_i\right\}
\]
and $G_p\rightarrow G$, for some well-behaved $G$, then can we treat $G$ like a prior distribution, and get the same asymptotic equivalence? We believe the answer is yes (and simulations support this). The mathematics for this problem becomes significantly more complex, and is beyond the scope of this manuscript.  Questions of this flavor (trying to treat the empirical distribution of the means as a prior) have been explored in the compound decision theory literature (though none that we have seen use this formal definition of frequentist selection bias).

\section{Extensions to General Problems}\label{sec:general}
An important strength of the bootstrap approach is that it is not restricted to this Gaussian scenario --- the bootstrap is flexible and can be applied to more complex problems intractable for empirical Bayes. Consider the more general scenario, where the $z$ vector has a joint distribution parametrized both by our parameters of interest $\m$, and some nuisance parameters $\thet$.
\[
z \sim F\left(\m,\thet\right)
\]
for some known functional form $F$. We can generalize our bias from before as
\[
\beta_k = \operatorname{E}_{F\left(\m,\thet\right)}\left[\hat{\mu}_{i(k)} - \mu_{i(k)}\right]
\]
where $\hat{\mu}_i$ is the MLE for $\mu_i$. The same risk reduction holds for estimates $\tilde{\mu}_{i(k)} = \hat{\mu}_{i(k)} - \beta_k$.

In this case we estimate the bias by
\[
\hat{\beta}\left(\m,\thet\right) = \beta\left(\hat{\m},\hat{\thet}\right)
\]
where $\hat{\m}$ and $\hat{\thet}$ are maximum likelihood estimates. We illustrate this on estimating $\rho^2$ values in regression with categorical variables in Section~\ref{sec:emp}, but it can be further applied to a wide variety of problems involving non-gaussian distributed statistics, dependence, and more complicated model-based estimates.

\section{Empirical Results}\label{sec:emp}
We give empirical results for the bootstrap method both for the simple gaussian scenario as well as the more complicated categorical variable regression problem. In the gaussian scenario, we compare the performance of the bootstrap to empirical Bayes methods, on real and simulated data. We see that both empirical Bayes and the bootstrap effectively reduce selection bias (in some cases quite drastically). The second order bootstrap outperforms the first order, and is comparable though somewhat outperformed by the best empirical Bayes estimates ({\tt nlpden}). However, these empirical Bayes methods cannot handle the categorical variable regression problem, while the bootstrap approach still performs well there.

\subsection{Simulated Results}\label{sec:sim}
We know that the potential gain of these procedures is based on the bias of our rank estimates, which in turn is based on the spacing of the means. We consider $6$ simulated scenarios with varying mean spacings to explore the bias estimation of the procedures in different regimes. 
\begin{enumerate}
\item All $\mu_i = 0$ --- hypothesis testing with a global null.
\item $500$ of the $\mu_i = 0$, and $500$ of the $\mu_i = 6$ --- a simple mixture model.
\item $900$ of the $\mu_i = 0$, and $100$ of the $\mu_i = 6$  --- hypothesis testing with a strong clustered set of alternatives
\item $900$ of the $\mu_i = 0$ and $100$ of the $\mu_i \sim N(0,2)$ --- hypothesis testing with a weak diffuse set of alternatives
\item All $\mu_i \sim N(0,1)$ --- diffuse means
\item $5$ clusters each with $200$ features. In the $j$-th cluster ($j=1,\ldots,5$) all $\mu_i = 6*j$ --- separated mixture model.
\end{enumerate}

Performance can be see in Table~\ref{tab:big}. While competitive with empirical Bayes, the bootstrap is outperformed by the nonparametric estimates of {\tt nlpden}. This table also illustrates the improvement from our second stage of bootstrap debiasing.

\begin{table}[ht]
\begin{center}
\begin{tabular}{rrrrrrr}
  \hline
 & 1 & 2 & 3 & 4 & 5 & 6 \\ 
  \hline
  boot1 & 0.10 & 0.22 & 0.17 & 0.24 & 0.54 & 0.30 \\ 
  boot2 &  \textcolor{blue}{0.04} &  \textcolor{blue}{0.12} &  \textcolor{blue}{0.09} &  \textcolor{blue}{0.20} &  \textcolor{blue}{0.53} &  \textcolor{blue}{0.18} \\ 
  spline3 & 0.02 & 0.16 & 0.19 & 0.19 & 0.53 & 1.00 \\ 
  spline5 & 0.04 & 0.10 & 0.09 & 0.19 & 0.53 & 1.00 \\ 
  spline7 & 0.09 & 0.09 & 0.07 & 0.20 & 0.56 & 0.96 \\ 
  {\tt nlpden} & \textcolor{MidnightBlue}{0.01} & \textcolor{MidnightBlue}{0.06} & \textcolor{MidnightBlue}{0.04} & \textcolor{MidnightBlue}{0.19} & \textcolor{MidnightBlue}{0.53} & \textcolor{MidnightBlue}{0.08} \\
  oracle & \textcolor{BrickRed}{0.002} &  \textcolor{BrickRed}{0.05} &  \textcolor{BrickRed}{0.03} &  \textcolor{BrickRed}{0.17} &  \textcolor{BrickRed}{0.52} &  \textcolor{BrickRed}{0.05}\\
   \hline
\end{tabular}
\end{center}
\caption{MSE, as a fraction of the MSE of the naive estimate, for each of the $6$ scenarios described (averaged over 20 trials). Boot1 and Boot2 are the (smoothed) one stage and two stage bootstrap estimates (With 100 bootstrap samples). spline3, 5,and 7 are empirical Bayes estimates using Lindsey's method with a spline basis and $3,\,5$, and $7$ df. Colored text is for emphasis and ease of reading.}
\label{tab:big}
\end{table}

We also have plots showing the shrinkage from the bootstrap and empirical Bayes for scenarios $1$, $3$,$4$,$6$ (from left to right and top to bottom). We can see in Plot~\ref{plot:shrink} that the bootstrap and empirical Bayes estimates are very similar. Though entirely unrealistic as an applied scenario, the results in scenario $6$ for both the bootstrap and {\tt nlpden} are particularly neat.

\begin{figure}
  \begin{center}
    \includegraphics[scale=0.9]{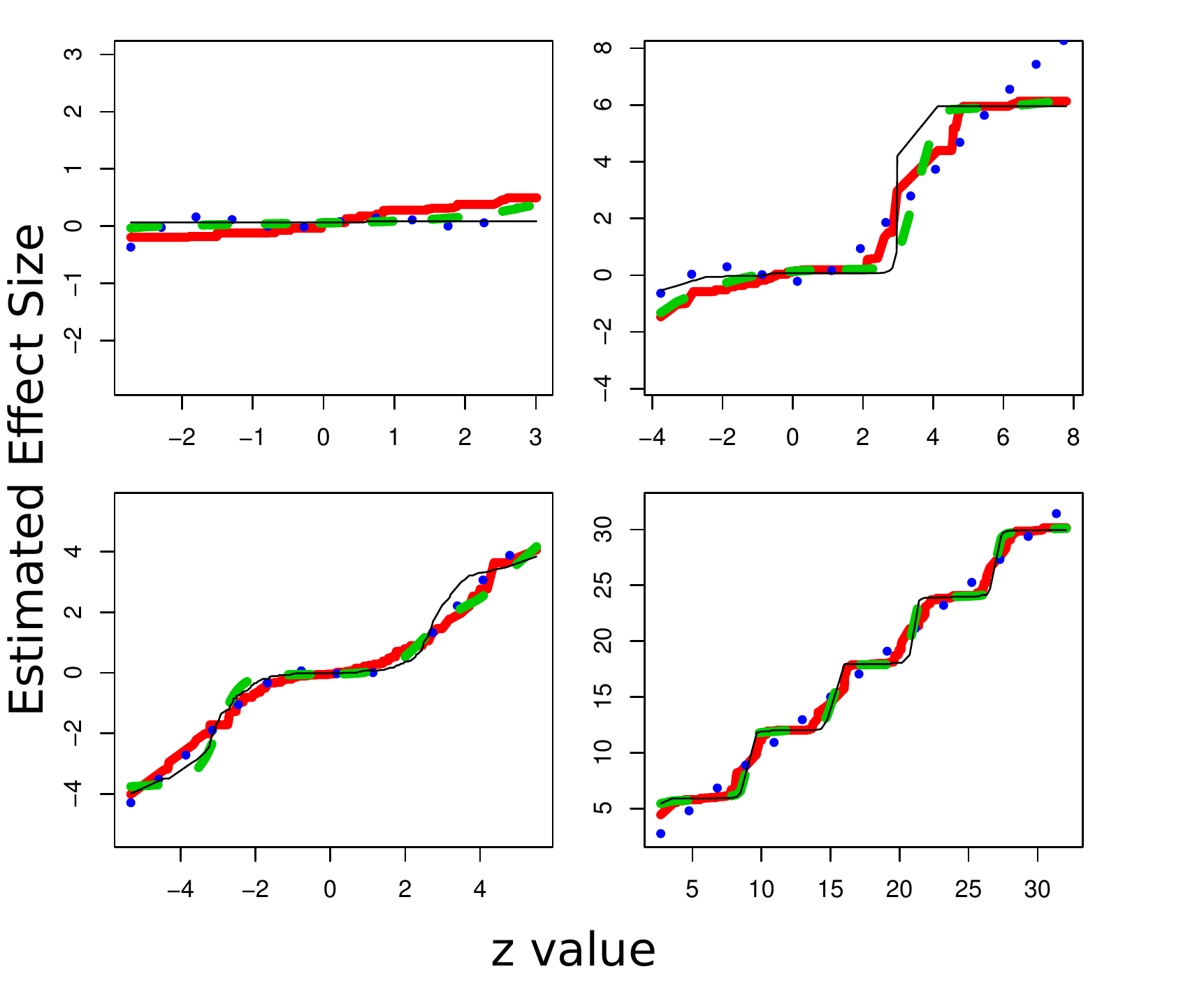}
  \end{center}
  \caption{Plots showing shrinkage, with ordered $z$-values on the x-axis vs estimated means on the y-axis. Plots are for scenarios $1$, $2$, $3$, and $6$. Blue dotted lines are estimates for Lindsey-estimated empirical Bayes (using a spline with 5 df), green dashed lines are estimates from {\tt nlpden}, and red solid lines are estimates from the (smoothed) second-order bootstrap. The skinny black line is the oracle estimate, using the biases calculated from the true means.}
  \label{plot:shrink}
\end{figure}

\subsection{Non-Gaussian Statistics}
One of the major strengths of our procedure is the ability to deal with nongaussian statistics. We now give a simulated example; estimating $\rho^2$, the coefficient of determination, in the regression of a continuous variable on a $3$-level categorical variable. We simulated an $n\times p$ matrix $X$ of features, and an $n\times p$ matrix $Y$ of responses (each feature had its own, response). The entries in each column of $X$ were sampled with equal probability from the $3$ classes (though we ensured at least $2$ observations per class in each column). The $\rho^2$ values for each regression were selected based on various schemes, and each entry of $Y$ was independently simulated as
\[
Y_{ij} = \sum_{k=0,1,2} I\{X_{ij} = k\} \theta_{jk} + \epsilon_{ij} 
\]
where $\epsilon_{ij} \sim N(0,1)$, and $\boldsymbol{\theta}_j = \{\theta_{j0},\theta_{j1},\theta_{j2}\}$ were chosen so that the regression had the prespecified $\rho^2$-value. We apply the general method (with nuisance parameters) detailed in Section~\ref{sec:general}. More specifically, we estimate our regression coefficients (and the variance of $\epsilon_{\cdot j}$), and calculate the bias from those estimated models.

For these examples we used $p=1000$ features with $n=50$ observations. The schemes used for choosing $\rho^2$ were:
\begin{enumerate}
\item All $\rho^2 = 0$
\item $\rho^2 \sim \operatorname{exponential}\left(10\right)$ (with values truncated at $0.99$)
\item $800$ of the $\rho^2 \sim \operatorname{exponential}\left(20\right)$, and $200$ $\rho^2 \sim N(0.55, 1/20)$ (again with truncation at $0.99$)
\end{enumerate}

Performance of the bootstrap methods on these examples can be seen in Table~\ref{tab:non-gaussian}. While perhaps not as extreme as the improvement in our gaussian examples, the bootstrap shrinkage does still show substantial gain over the naive estimates. We illustrate the shrinkage plots for scenarios $2$ and $3$ in Figure~\ref{fig:catShrink}. As we see from both Table~\ref{tab:non-gaussian} and Figure~\ref{fig:catShrink}, our estimates are close to the oracle estimates.

\begin{table}[ht]
\begin{center}
\begin{tabular}{rrrr}
  \hline
 & 1 & 2 & 3\\ 
  \hline
  boot1 & \textcolor{blue}{0.0516} & 0.550 & 0.511\\
  boot2 &  0.061 &  \textcolor{blue}{0.546} &  \textcolor{blue}{0.475}\\
  oracle & \textcolor{red}{0.002} &  \textcolor{red}{0.538} &  \textcolor{red}{0.442}\\
   \hline
\end{tabular}
\end{center}
\caption{MSE, as a fraction of the MSE of the naive estimate, for each of the $3$ scenarios described (averaged over 20 trials). Boot1 and Boot2 are the (smoothed) one stage and two stage bootstrap estimates (With 100 bootstrap samples). ``Oracle'' is the oracle estimate, using the biases calculated from the true model.}
\label{tab:non-gaussian}
\end{table}

\begin{figure}
  \begin{center}
    \includegraphics[scale=0.4]{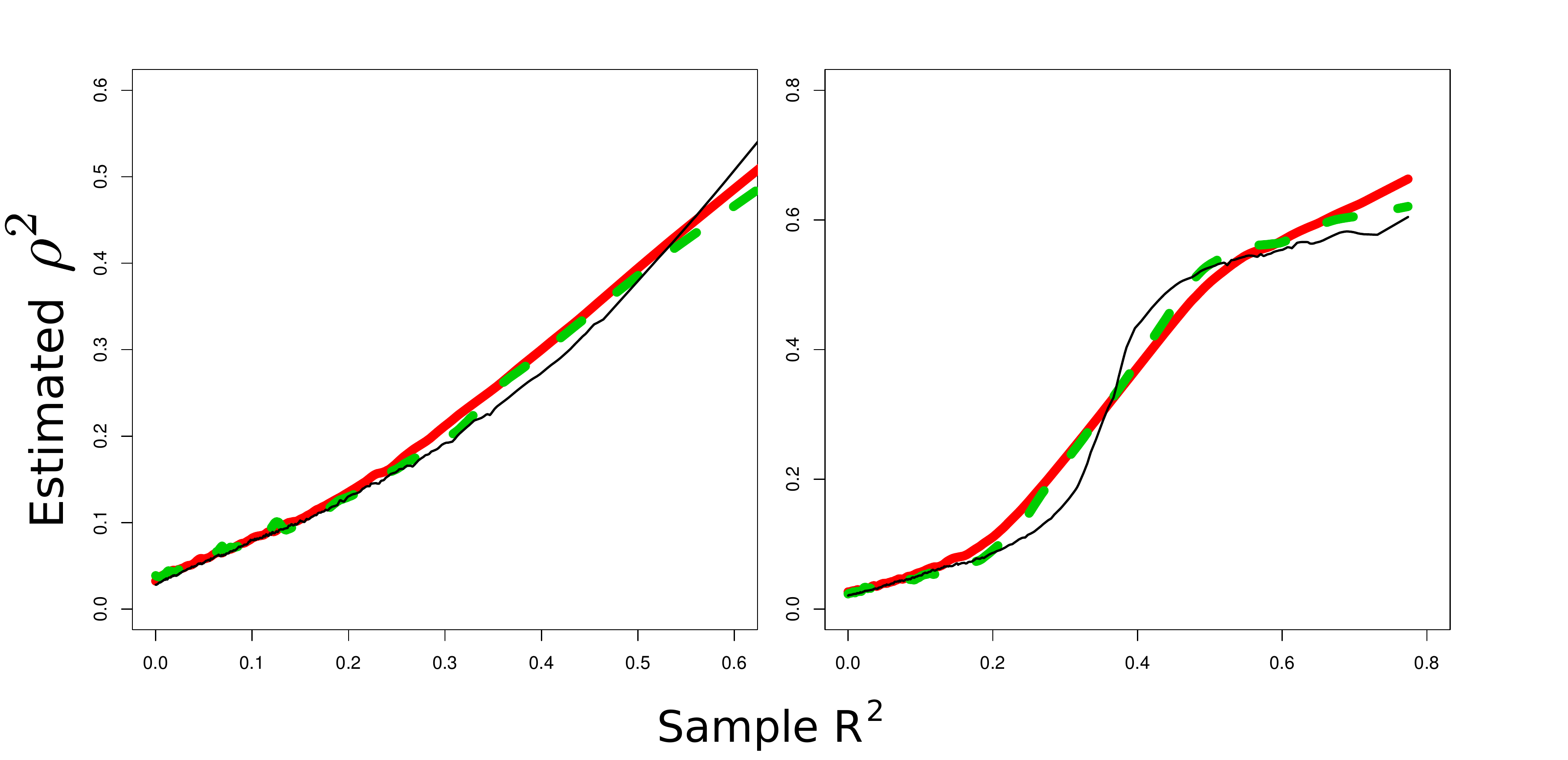}
  \end{center}
  \caption{Plots showing shrinkage for categorical simulations (scenario $2$ on the left, $3$ on the right), with ordered sample $R^2$-values on the x-axis vs estimates of $\rho^2$ on the y-axis. The solid red and dashed green lines give the (smoothed) first and second order bootstrap estimates of $\rho^2$ (based on $100$ bootstrap samples). The skinny black line is the oracle estimate, using the bias calculated from the true model.}
  \label{fig:catShrink}
\end{figure}

\subsection{Real Data}
While simulated data can give insight, it can also be misleading as the efect-sizes for real data rarely follow our simplified simulation schemes. To compare the behavior of the bootstrap and empirical Bayes approaches on real data, we applied both methods to the prostate data of \citet{singh2002}. This data has $102$ samples, $52$ from prostate cancer tumors and $50$ from healthy tissue. For each sample, there are $6033$ gene expression measurements. For each of these $6033$ features we calculated a two sample $t$-statistic (which we transformed to be approximately normal under the null).\\

As shown in Figure~\ref{fig:pros} our bootstrap and empirical Bayes estimates are similar, though the empirical Bayes methods provide more shrinkage in the center. Unfortunately there is no gold standard for this data, so it is hard to compare the performance of the methods beyond this.

\begin{figure}
  \begin{center}
    \includegraphics[scale=0.7]{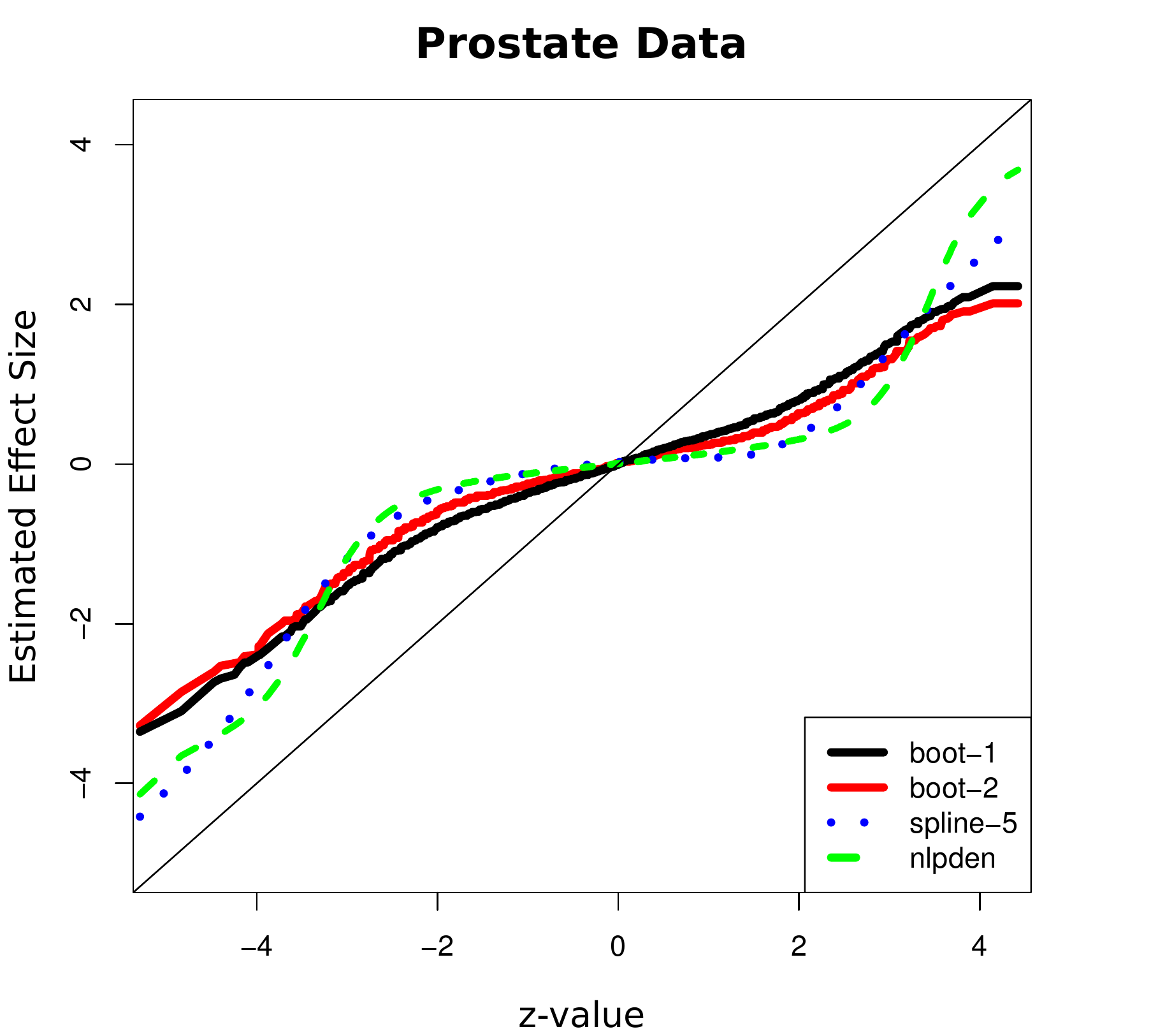}
  \end{center}
  \caption{Plot showing shrinkage for prostate data, with ordered $z$-values on the x-axis vs estimated means on the y-axis. Dotted and dashed lines are empirical Bayes estimates. Solid lines are first-order and second-order bootstrap estimates.}
  \label{fig:pros}
\end{figure}

\section{Discussion}
Large scale estimation of effect-sizes has become of increased importance, in particular in biomedical and epidemiological studies with the development of whole genome assays like gene expression microarrays, genotyping arrays and next generation DNA sequencing. Typical studies attempt to identify univariate differences between disease cases and control. Methods for control of false discoveries are commonly used in such studies, but the effect sizes of the features selected as important are rarely adjusted for selection. This can result in a misleading assessment of the importance of the feature and an overestimation of its value for predictive purposes.\\

In light of this, shrinkage estimators are becoming more and more important. Classical global shrinkage (as in James-Stein type estimators) unfortunately tends to overshrink the estimated size of the most extreme (and therefore interesting) effects. Modern shrinkage estimators (such as our proposal and non-parametric empirical Bayes) leverage local information to give data-adaptive estimates, with much better performance for the effect-sizes of ``interesting'' effects.

In this paper we give a general formalism for frequentist selection bias. We show how this formalism connects to Bayesian ideas for removing selection bias. We propose a resampling estimator for estimating this frequentist selection bias which leads to accurate estimates of effect-size based on locally adaptive shrinkage. Unlike empirical Bayes methods which must be hand-crafted for each scenario, our approach is general and applicable in a wide array of problems. Our method is widely applicable, simple to apply, and performs well in practice.

\section{Appendix}

Here we will give a proof for Theorem~\ref{theorem:Bayes}.
\begin{proof}[Proof of Theorem~\ref{theorem:Bayes}]
To begin we note that sample quantiles converge to population quantiles, or
\[
z_{[\lfloor pt \rfloor]} \overset{p}{\rightarrow} q_t,
\]
so, because $G$ has finite support, we have that: $\operatorname{E}\left[z_{[\lfloor pt \rfloor]}\right] \rightarrow q_t$.\\

Continuing to the second term, using the fact that $\mu_i$ are iid, we have
\begin{align*}
\operatorname{E}\left[\mu_{i(\lfloor pt \rfloor)}\right] &=\sum_{j=1}^p\operatorname{E}\left[\mu_j \middle |\, i\left(\lfloor pt \rfloor\right) = j\right] \operatorname{P}\left(i\left(\lfloor pt \rfloor\right) = j\right)\\
&= \operatorname{E}\left[\mu_1 \middle|\, i\left(\lfloor pt \rfloor\right) = 1\right]
\end{align*}
Further, we see that
\begin{align*}
\operatorname{E}\left[\mu_1 \middle|\, i\left(\lfloor pt \rfloor\right) = 1\right] &= \int_{\mu_1} \mu_1 dP\left(\mu_1 \middle |\, i\left(\lfloor pt \rfloor\right) = 1\right)\\
&= \int_{z_1}\int_{\mu_1} \mu_1\, dP\left(\mu_1,\, z_1 \middle |\, i\left(\lfloor pt \rfloor\right) = 1\right)\\
&= \int_{z_1}\int_{\mu_1} \mu_1\, dP\left(\mu_1 \middle |\, i\left(\lfloor pt \rfloor\right) = 1,\, z_1\right)dP\left(z_1 \middle |\, i\left(\lfloor pt \rfloor\right) = 1\right)\\
&= \int_{z_1}\int_{\mu_1} \mu_1\, \frac{\operatorname{P}\left( i\left(\lfloor pt \rfloor\right) = 1 \middle |\,\mu_1,\, z_1\right)} 
                                  {\operatorname{P}\left( i\left(\lfloor pt \rfloor\right) = 1 \middle |\, z_1\right)}\cdot 
                                  dP\left(\mu_1 \middle |\, z_1\right)\cdot
                                  dP\left(z_1 \middle |\, i\left(\lfloor pt \rfloor\right) = 1\right)\\
\end{align*}
but note that
\[
\operatorname{P}\left( i\left(\lfloor pt \rfloor\right) = 1 \middle |\,\mu_1,\, z_1\right) = \operatorname{P}\left( i\left(\lfloor pt \rfloor\right) = 1 \middle |\, z_1\right)
\]
so we get
\begin{align*}
\operatorname{E}\left[\mu_1 \middle|\, i\left(\lfloor pt \rfloor\right) = 1\right] &= \int_{z_1}\int_{\mu_1} \mu_1\, dP\left(\mu_1 \middle |\, z_1\right)
                                  ) dP\left(z_1 \middle |\, i\left(\lfloor pt \rfloor\right) = 1\right)\\
&= \int_{z_1}\operatorname{E}\left[\mu_1 \middle |\, z_1\right]dP\left(z_1 \middle |\, i\left(\lfloor pt \rfloor\right) = 1\right)\\
&\rightarrow \operatorname{E}\left[\mu_1 \middle |\, z_1 = q_t\right].
\end{align*}
The last line follows through concentration of measure (combined with the finite support of $G$), because, as we have shown, $z_{[\lfloor pt \rfloor]} \rightarrow q_t$, thus completing the proof\hfill $\blacksquare$

\end{proof}

\bibliographystyle{abbrvnat}
\bibliography{man}

\begin{thebibliography}{12}
\providecommand{\natexlab}[1]{#1}
\providecommand{\url}[1]{\texttt{#1}}
\expandafter\ifx\csname urlstyle\endcsname\relax
  \providecommand{\doi}[1]{doi: #1}\else
  \providecommand{\doi}{doi: \begingroup \urlstyle{rm}\Url}\fi

\bibitem[Brown and Greenshtein(2009)]{brown2009}
L.~D. Brown and E.~Greenshtein.
\newblock Nonparametric empirical bayes and compound decision approaches to
  estimation of a high-dimensional vector of normal means.
\newblock \emph{The Annals of Statistics}, pages 1685--1704, 2009.

\bibitem[Efron(2011)]{efron2011}
B.~Efron.
\newblock Tweedie's formula and selection bias.
\newblock \emph{Journal of the American Statistical Association}, 106\penalty0
  (496):\penalty0 1602--1614, 2011.

\bibitem[Efron and Morris(1973)]{efron1973}
B.~Efron and C.~Morris.
\newblock Stein's estimation rule and its competitors -- an empirical bayes
  approach.
\newblock \emph{Journal of the American Statistical Association}, 68\penalty0
  (341):\penalty0 117--130, 1973.

\bibitem[Efron and Tibshirani(1996)]{efron1996}
B.~Efron and R.~Tibshirani.
\newblock Using specially designed exponential families for density estimation.
\newblock \emph{The Annals of Statistics}, 24\penalty0 (6):\penalty0
  2431--2461, 1996.

\bibitem[James and Stein(1961)]{stein1961}
W.~James and C.~Stein.
\newblock Estimation with quadratic loss.
\newblock In \emph{Proceedings of the fourth Berkeley symposium on mathematical
  statistics and probability}, pages 361--379, 1961.

\bibitem[Jiang and Zhang(2009)]{jiang2009}
W.~Jiang and C.-H. Zhang.
\newblock General maximum likelihood empirical bayes estimation of normal
  means.
\newblock \emph{The Annals of Statistics}, 37\penalty0 (4):\penalty0
  1647--1684, 2009.

\bibitem[Robbins(1956)]{robbins1956}
H.~Robbins.
\newblock Asymptotically subminimax solutions of compound statistical decision
  problems.
\newblock In \emph{Proceedings of the Third Berkeley symposium on mathematical
  statistics and probability}, pages 157--164, 1956.

\bibitem[Robbins(1985)]{robbins1985}
H.~Robbins.
\newblock \emph{An Empirical Bayes Approach to Statistics}.
\newblock Springer, 1985.

\bibitem[Singh et~al.(2002)Singh, Febbo, Ross, Jackson, Manola, Ladd, Tamayo,
  Renshaw, D'Amico, Richie, et~al.]{singh2002}
D.~Singh, P.~G. Febbo, K.~Ross, D.~G. Jackson, J.~Manola, C.~Ladd, P.~Tamayo,
  A.~A. Renshaw, A.~V. D'Amico, J.~P. Richie, et~al.
\newblock Gene expression correlates of clinical prostate cancer behavior.
\newblock \emph{Cancer cell}, 1\penalty0 (2):\penalty0 203--209, 2002.

\bibitem[Stein(1956)]{stein1956}
C.~Stein.
\newblock Inadmissibility of the usual estimator for the mean of a multivariate
  normal distribution.
\newblock In \emph{Proceedings of the Third Berkeley symposium on mathematical
  statistics and probability}, pages 197--206, 1956.

\bibitem[Tusher et~al.(2001)Tusher, Tibshirani, and Chu]{tusher2001}
V.~G. Tusher, R.~Tibshirani, and G.~Chu.
\newblock Significance analysis of microarrays applied to the ionizing
  radiation response.
\newblock \emph{Proceedings of the National Academy of Sciences}, 98\penalty0
  (9):\penalty0 5116--5121, 2001.

\bibitem[Wager(2013)]{wager2013}
S.~Wager.
\newblock A geometric approach to density estimation with additive noise.
\newblock \emph{Statistica Sinica}, 2013.

\end{thebibliography}

\end{document}